\documentclass[conference]{IEEEtran}
\IEEEoverridecommandlockouts
\usepackage{cite}
\usepackage{amsmath,amssymb,amsfonts}
\usepackage{graphicx}
\usepackage{textcomp}
\usepackage{xcolor}
\def\BibTeX{{\rm B\kern-.05em{\sc i\kern-.025em b}\kern-.08em
    T\kern-.1667em\lower.7ex\hbox{E}\kern-.125emX}}

\usepackage{times}
\usepackage{soul}
\usepackage{url}
\usepackage[hidelinks]{hyperref}
\usepackage[utf8]{inputenc}
\usepackage[small]{caption}
\usepackage{graphicx}
\usepackage{amsmath}
\usepackage{amsthm}
\usepackage{booktabs}
\usepackage{algorithm}
\usepackage{algpseudocode}

\usepackage{amssymb}
\usepackage{mathrsfs}
\usepackage{enumerate}
\usepackage{cleveref}
\usepackage{fancyhdr}
\usepackage{graphicx}
\usepackage{subfigure}

\newcommand{\R}{\mathbb{R}}
\def \real    { \mathbb{R} }

\newcommand{\e}{\begin{equation}}
\newcommand{\ee}{\end{equation}}
\newcommand{\en}{\begin{equation*}}
\newcommand{\een}{\end{equation*}}
\newcommand{\eqn}{\begin{eqnarray}}
\newcommand{\eeqn}{\end{eqnarray}}
\newcommand{\bmat}{\begin{bmatrix}}
\newcommand{\emat}{\end{bmatrix}}
\newcommand{\vct}[1]{\boldsymbol{#1}}
\newcommand{\mtx}[1]{\boldsymbol{#1}}
\newcommand{\T}{\mathrm{T}}

\newcommand{\trace}{\operatorname{trace}}

\newcommand{\set}[1]{\mathbb{#1}}
\DeclareMathOperator*{\argmin}{\text{arg~min}}

\newcommand{\ve}{\vct{e}}

\newcommand{\vh}{\vct{h}}

\newcommand{\vp}{\vct{p}}

\newcommand{\vu}{\vct{u}}
\newcommand{\vv}{\vct{v}}

\newcommand{\vx}{\vct{x}}
\newcommand{\vy}{\vct{y}}

\newcommand{\mA}{\mtx{A}}
\newcommand{\mB}{\mtx{B}}

\newcommand{\mD}{\mtx{D}}

\newcommand{\mH}{\mtx{H}}
\newcommand{\mI}{\mtx{I}}

\newcommand{\mL}{\mtx{L}}
\newcommand{\mM}{\mtx{M}}

\newcommand{\mP}{\mtx{P}}
\newcommand{\mQ}{\mtx{Q}}
\newcommand{\mR}{\mtx{R}}

\newcommand{\mU}{\mtx{U}}
\newcommand{\mV}{\mtx{V}}

\newcommand{\mX}{\mtx{X}}

\newcommand{\mSigma}{\mtx{\Sigma}}

\newcommand{\mId}{{\bf I}}

\newcommand{\setH}{\set{H}}

\graphicspath{{./figs/}}

\newlength{\imgwidth}
\newboolean{twoColVersion}
\setboolean{twoColVersion}{false}
\newcommand{\twoCol}[2]{\ifthenelse{\boolean{twoColVersion}} {#1} {#2} }

\def\@onedot{\ifx\@let@token.\else.\null\fi\xspace}

\def\etc{\emph{etc.}} 
\def\wrt{w.r.t.} 

\newtheorem{prop}{Proposition}
\newtheorem{lem}{Lemma}
\newtheorem{thm}{Theorem}%[section] %(If you want theorem numbered

\newtheorem{dis}{Discussion}
\newtheorem{cor}{Corollary}%[section]
\begin{document}

\title{Rethinking Symmetric Matrix Factorization: \\A More General and Better Clustering Perspective}
%\author{\IEEEauthorblockN{Mengyuan Zhang, Kai Liu}}

\author{
	\IEEEauthorblockN{\ Mengyuan Zhang}
	\IEEEauthorblockA{
		\textit{Clemson University}\\
	\ 	Clemson, SC, USA\\
	\	\ mengyuz@clemson.edu
	}\and
	\IEEEauthorblockN{\ Kai Liu}
	\IEEEauthorblockA{ 
		\textit{Clemson University}\\
		\ \ Clemson, SC, USA\\
		\ \ \ kail@clemson.edu
	}
}

\maketitle

\begin{abstract}
  Nonnegative matrix factorization (NMF) is widely used for clustering with strong interpretability. Among general NMF problems, symmetric NMF is a special one that plays an important role in graph clustering where each element measures the similarity between data points. Most existing symmetric NMF algorithms require factor matrices to be nonnegative, and only focus on minimizing the gap between similarity matrix and its approximation for clustering, without giving a consideration to other potential regularization terms which can yield better clustering. In this paper, we explore factorizing a symmetric matrix that does not have to be nonnegative, presenting an efficient factorization algorithm with a regularization term to boost the clustering performance. Moreover, a more general framework is proposed to solve symmetric matrix factorization problems with different constraints on the factor matrices.
\end{abstract}


\section{Introduction}

Nonnegative matrix factorization (NMF) \cite{lee1999learning} problem is formulated as the following: given a data matrix $\mX = [\vx_1, \vx_2, \dots, \vx_n] \in \real^{m\times n}_{+}$ containing $n$ observations, each observation denoted as $\vx_i$ is an $m$ dimensional vector, where $\real^{m\times n}_{+}$ denotes the set of $m \times n$ element-wise nonnegative matrices. NMF aims to find a lower-rank matrix approximation represented by:
\begin{equation}
	\label{eq:approx}
	\mX \approx \mU\mV^T,
\end{equation}
where $\mU = [\vu_1, \vu_2, \dots, \vu_k] \in \real^{m\times k}_{+}$, and $\mV = [\vv_1, \vv_2, \dots, \vv_k] \in \real^{n\times k}_{+}$. Typically squared Frobenius norm is used to measure the distance between $\mX$ and $\mU\mV^T$, so the objective in NMF is formulated as~\cite{kim2008nonnegative}:
\begin{equation}
	\label{eq:nmf}
	\min_{\mU,\mV \geq 0} \|\mX-\mU\mV^T\|^2_F,
\end{equation}
where $\| \mQ \|_F^2=\sum_{i,j}\mQ(i,j)^2$. Usually, $k$ is assumed to be  smaller than $\min \{m,n\}$, thus NMF can be regarded as a lower-rank approximation problem. 

Apparently, the NMF paradigm described above conducts clustering based on input data directly and assumes data is well linearly separable. However, for data that lies in a specific manifold (say a certain sphere or two moons), it will yield a poor result. Therefore, graph clustering is introduced to overcome the difficulty based on a matrix that measures the similarity between each data point~\cite{kuang2012symmetric}. The factorization of similarity matrix $\mA\in\real^{n\times n}_{+}$ will yield a lower-rank matrix $\mH\in\real^{n\times k}_{+}$ which plays a similar role as $\mV$ for cluster assignment~\cite{kuang2012symmetric,he2011symmetric,liu2021factor,liu2018multiple}. Specifically, symmetric NMF formulates the objective as: %In clustering tasks, the symmetric NMF of the similarity matrix $\mA$ can be addressed as follows: given an $n \times n$ similarity matrix $\mA$, where $n$ is the number of data points, the goal is to: 
\begin{equation}
	\label{eq:snmf}
	\min_{\mH \geq 0} \|\mA-\mH\mH^T\|^2_F,
\end{equation}
where $k$ is the number of clusters. Compared to classical NMF, symmetric NMF is more flexible in terms of admitting any reasonable  measurement with mixed signs such as cosine similarity~\cite{wu2018pairwise,zhao2020snmfsmma,liu2018high}.%any similarity measures that can well capture the inherent cluster structures can be utilized. 

Previous work on symmetric NMF mostly requires that the matrix $\mH$ is nonnegative.  Therefore, even $\mA$ is not explicitly constrained to be non-negative, in practice, it is equivalent to setting the negative elements to be 0. When symmetric NMF is applied to graph clustering, the  result is directly obtained from $\mH$ while ignoring some other techniques such as graph regularization to promote clustering performance.
%Previous work on symmetric NMF mostly requires that the matrix $\mA$ is nonnegative and symmetric positive semidefinite\cite{catral2004reduced, ding2005equivalence, zhu2018dropping}, to overcome this limitation and extend symmetric NMF to a general graph clustering model, Symmetric NMF algorithm (SymNMF) was proposed\cite{kuang2012symmetric, kuang2015symnmf}. In SymNMF, $\mA$ is not required to be nonnegative but the nonnegative constraint on the lower-rank matrix $\mH$ is still emphasized.
In this paper, we comprehensively study symmetric matrix factorization with its application in graph clustering. Our contribution is threefold:
\begin{itemize}
	\item We first extend vanilla symmetric NMF and study a more general case where there is no non-negative constraint on $\mH$ and interpret it in a meaningful manner. We propose a very efficient updating algorithm that can be extended to the non-negative case.
    \item A regularization term is added to boost the clustering performance.
	Instead of merely focusing on minimizing the objective $\|\mA-\mH\mH^T\|^2_F$, we impose graph regularization term to ensure that data points with higher similarity value share more similar cluster indicators, and vice versa, with  $\mH$  admitting mixed signs.%Different from other manifold-based methods, we utilize inner product similarity matrix rather than connected graph built with extra information.
	\item We propose a general framework that can deal with symmetric matrix factorization problems with various constraints by learning the stepsize adaptively. %Different from Projected Gradient Descent method, we learn the adaptive stepsize and present its convergence property. %Our method is different from those with fixing stepsize which can either be too large or too small, while on the other hand, it is different from line search method which is time consuming. %The algorithm itself does not depend on any further constraints on the similarity matrix $\mA$ or the lower-rank matrix $\mH$. While the algorithm already performs well in clustering tasks without any constraint on the matrix $\mH$, we can still apply the algorithm to solve problems with constrains on $\mH$, such as nonnegative and sparseness constraint.
\end{itemize}

This paper is organized as follows: 
Section \ref{sec:motivation} describes the motivation behind the problem. Section \ref{sec:formulation} provides the problem formulation and optimization methods, including one column-wise update method with fast speed, and one more general framework. Section \ref{sec:convergence} shows the convergence rate of our method. Section \ref{sec:exp} reports the experimental results on both image datasets and text datasets, followed by the Conclusion Section. Our code is available on GitHub.\footnote{\url{https://github.com/clair-lab/Symmetric-Matrix-Factorization}}

\section{Motivation}\label{sec:motivation}
\begin{figure}
	\centering
	\includegraphics[width=.9\linewidth]{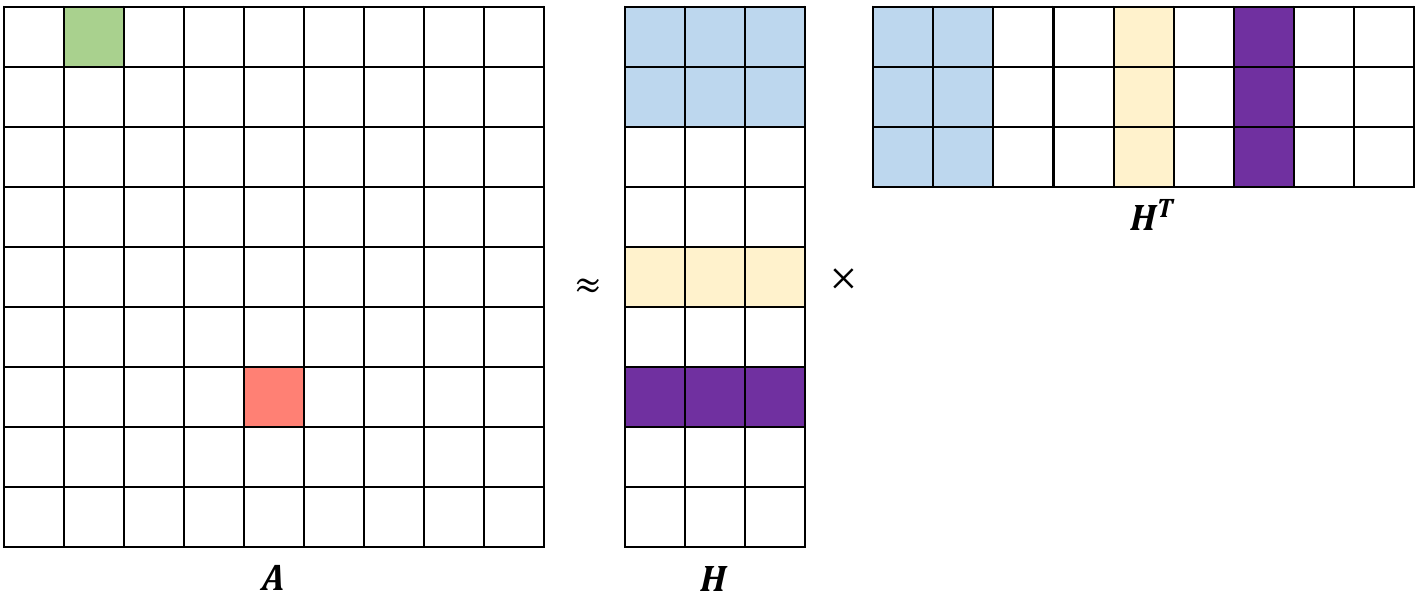}
	\caption{An illustration of symmetric matrix factorization on $\mA\in\mathbb{S}^n$. The green block indicates the similarity is high, the corresponding rows (first two) in $\mH$ should also be similar, so they are filled with the same color. The red block indicates the similarity is low, where the corresponding rows (5th and 7th, respectively) in $\mH$ should be dissimilar, filled with different colors. }
	\label{fig:matrix}
\end{figure}

In symmetric NMF for clustering, the objective function (\ref{eq:snmf}) is to measure the gap between the original similarity matrix $\mA$ and $\mH\mH^T$, where $\mH$ is the clustering assignment matrix with nonnegative constraint. However, most algorithms only aim to minimize the gap $\|\mA-\mH\mH^T\|^2_F$ while ignoring the potential over-fitting which may lower the clustering performance. Following the idea in graph regularization~\cite{cai2010graph,gu2009co}: data points that have high similarity (in $\mA$) should have closer clustering indicators (rows of $\mH$), and vice versa, which is demonstrated as Fig.~\ref{fig:matrix}. Accordingly, the regularization term is given by:%If a term reflecting this principle can be incorporated in the objective function to avoid overfitting $\min\|\mA-\mH\mH^T\|^2_F$, 
\begin{equation}
	\label{eq:regu}
	\min_{\mH \geq 0} \sum_{i,j = 1}^n \mA_{ij} \|\vh^i-\vh^j\|_2^2,
\end{equation}
where $\vh^i$ denotes $i$-th row in $\mH$.
%where $\mA$ is the similarity matrix of size $n \times n$, $\mH$ is the lower-rank matrix of size $n \times k$, and $\vh^i$ denotes the $i$-th row in $\mH$.

%Intuitively, one can expect a balance between good approximation $\mA \approx \mH\mH^T$ and better clustering performance.  The foundation of such regularization method is usually referred to as manifold assumption~\cite{10.5555/2980539.2980616}. However, most clustering methods based on the manifold assumption need to build a connected graph using high-cost methods like $K$-NN, and require lower-rank matrix factors to be nonnegative. In our method the similarity matrix is built with the easiest way: inner product, and we can apply kernel to deal with more complex situations. 

%Most of existing Symmetric NMF formulate the objective as Eq. (\ref{eq:snmf}). 
Though theoretically, $\mA$ admits mixed signs, however, due to the nonnegative constraint on $\mH$, negative elements are treated as 0 after projection and play no role in learning.
%
%it is suggested transforming a similarity matrix with mixed
%signs into a nonnegative one before applying SymNMF~\cite{kuang2012symmetric}, where negative element will be set 0. That said, the drawback is obvious: for negative similarity element in $\mA$, if it is set 0, then it plays no role in regularization as whatever gap between the indicators is, the output of multiplication is always 0. 
Therefore, we remove the nonnegative constraint on $\mH$. Naturally, if $\mA_{ij}$ is negative, the indicators should be significantly different and $\mA_{ij} \|\vh^i-\vh^j\|_2^2$ remains negatively small, which is in accordance with the spirit of graph regularization. %Therefore, the newly added term can then indeed help improve clustering performance by removing the constraint. Regarding $\mH$, which is the membership indicator matrix for clustering, 
Negative element $h_{ij}$ denotes the \textbf{unlikelihood} of $i$-th data belonging to $j$-th cluster while positive represents the very \textbf{likelihood}. Therefore we formulate the objective with regularization as:
\begin{equation}
	\label{eq:smf}
	\min_{\mH}\|\mA-\mH\mH^T\|^2_F + \lambda \sum_{i,j = 1}^n \mA_{ij}\|\vh^i-\vh^j\|_2^2,
\end{equation}
where $\lambda$ is the tuning regularization parameter.% that decides how much we want to regularize the matrix factorization fitting with the extra term.

To verify whether the regularization term can help to boost the clustering performance, a pilot experiment is conducted on  COIL-20 data set~\cite{coil20}. Fig.~\ref{fig:reg} shows clustering accuracy~\cite{10.1145/860435.860485} comparison obtained from the same initialization and constraint (admitting mixed signs on $\mH$) with the only difference being the existence of the regularization term. We can see that the clustering performance is significantly improved by incorporating the regularization term into the objective, especially when the number of clusters grows larger.
\begin{figure}
	\centering
	\includegraphics[width=.8\linewidth]{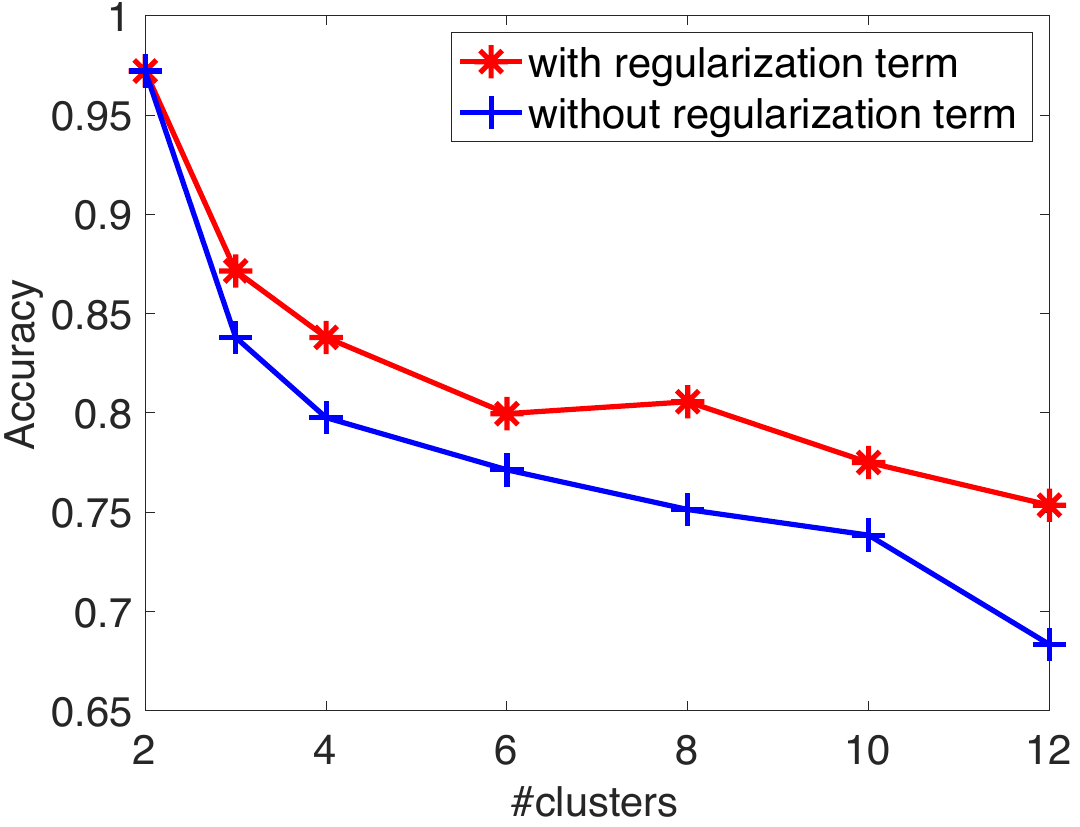}
	\caption{Accuracy of clustering comparison on COIL20 data set with different cluster numbers.}
	\label{fig:reg}
\end{figure}

Moreover, regarding $\mH$, there can be different constraints on it. One can verify that if $\mH^T\mH=\mI$, then it becomes spectral clustering~\cite{ng2002spectral,bach2004learning}; if all is zero except one element is 1 in each row of $\mH$, then it is $K$-means; if each row is nonnegative and the sum is 1~\cite{li2006relationships,schmidt2009bayesian}, then it is the probability distribution to each cluster. Instead of proposing ad-hoc algorithms with various constraints, we will systematically address the problem by providing a general framework that can obtain optimal solutions with updates efficiently.

\section{Formulation and Algorithm}\label{sec:formulation}
%\subsection{Objective Function with Regularization Term}
%The previous section indicates that: for graph clustering, we can find a lower-rank matrix $\mH$ to approximate $\mA$ with $\mH\mH^T$ while maintaining $\sum_{i,j} \mA_{ij}\|\vh^i-\vh^j\|_2^2$ relatively small (another interpretation is to penalize cluster assignment inconsistency). As demonstrated earlier, the objective function in our paper is to:
%\begin{equation}
%	\label{eq:smfo}
%	\min_{\mH} \|\mA-\mH\mH^T\|^2_F + \lambda \sum_{i,j = 1}^n \mA_{ij}\|\vh^i-\vh^j\|_2^2.
%\end{equation}
%where $\mA$ is similarity matrix $\in \real^{n \times n}$, $\mH$ is a lower-rank matrix $\in \real^{n \times k}$ and $\vh^i$ denotes the $i$-th row of $\mH$, $\lambda$ is the tuning regularization parameter.
In this section, we propose our algorithm to solve (\ref{eq:smf}) which enjoys faster convergence than its existing counterparts. 
\subsection{Reformulation}
By noting the regularization term is in row-wise form, we are to reformulate the objective in a more compact way as:
\begin{thm}\label{thm:reformulation}
(\ref{eq:smf}) 
is equivalent to:
\begin{equation}
	\label{eq:sim}
	\min_{\mH} \|\mA-\lambda \mL - \mH\mH^T\|^2_F,
\end{equation}
where $\mL$ is the Laplacian matrix given by $\mL := \mD-\mA \in \real^{n \times n}$ and 
$\mD$ is the degree matrix, which is diagonal defined as: $\mD(i,i)=\sum_j(A(i,j))$.
\end{thm}

\begin{proof}[Proof of Theorem \ref{thm:reformulation}]
Following \cite{cai2010graph}, we have:
%\begin{lem}
%For arbitrary column vector $\vq$, we have
%\begin{equation}\label{eq:lem}
%\begin{aligned}
%\vq^T \mL \vq &= \frac{1}{2}\sum_{i,j=1}^n \mA_{ij}(q_i - q_j)^2.
%\end{aligned}
%\end{equation}
%\end{lem}
%We use $\vh_c$ to denote the $c$-th column in $\mH$, $\vh^i$ to denote the $i$-th row in $\mH$, $h_{ic}$ to denote the $c$-th element in the $i$-th row of $\mH$. 
\begin{equation}\label{eq:p1}
%\begin{aligned}
\sum_{i,j=1}^n \mA_{ij}\|\vh^i-\vh^j\|_2^2 =2 tr(\mH^T\mL\mH).
%\end{aligned}
\end{equation}
Therefore, (\ref{eq:smf}) is equivalent to
\begin{equation}
	\label{eq:p2}
	\min_{\mH} \|\mA- \mH\mH^T\|^2_F + 2\lambda tr(\mH^T\mL\mH).
\end{equation}
By expanding (\ref{eq:sim}), we obtain:
\begin{equation}\label{eq:p3}
\begin{split}
&\min_{\mH} \|\mA-\lambda \mL- \mH \mH^T\|^2_F\\
 =&\min_{\mH} tr(\mA-\lambda \mL - \mH \mH^T)^T(\mA-\lambda \mL - \mH \mH^T)\\
 =&\min_{\mH} \|\mA- \mH \mH^T\|^2_F + 2\lambda tr(\mH^T \mL \mH)\\
  &+\textit{ terms irrelevant to }\mH.
\end{split}
\end{equation}
Thus we conclude that solution to (\ref{eq:smf}) is the same as to (\ref{eq:sim}) in terms of optimizing $\mH$.
\end{proof}
%With the above proof, we will use Eq.~(\ref{eq:sim}) as our objective function in the following sections.

\subsection{Column-Wise Fast Update}
Given the new formulation in (\ref{eq:sim}), we now turn to provide detailed updating rule for $\mH$. 

We first denote $\mA - \lambda \mL$ as $\mM$, obviously $\mM$ is still symmetric, though not necessarily positive definite. The optimization problem thereafter can be reformulated as: 
\begin{equation}
	\label{eq:m}
	\min_{\mH} f(\mH)=\|\mM- \mH\mH^T\|^2_F.
\end{equation}
The above problem, theoretically speaking, has closed solutions, though not unique (unless $\mM$ is negative semi-definite). One can see that if $\mH$ is an optimal solution, then $\mH\mR$ admits the same objective as long as $\mR\mR^T=\mI$ ($\mR\in \real^{k \times k}$). Apparently, if $\mM$ is positive definite, from Eckart-Young-Mirsky theorem, $\mH(:,i)=\pm\sqrt{\sigma_i}\vv_i$ will obtain the minimal where $\sigma_i$ is the top $i$-th eigenvalue and $\vv_i$ is the corresponding eigenvector of $\mM$. When $\mM$ is negative definite, then $\mH^*=0$ is unique solution. If $\mM$ has mixed signs in the eigenvalues, %WLOG, denoted by $[\lambda_1,\lambda_2,\dots,\lambda_t,\dots,\lambda_n]$ t
then $\mH(:,i)=\pm\sqrt{\max\{\sigma_i,0\}}\vv_i$. However, one significant disadvantage of the above method is when the size of $\mM$ is very large, conducting eigenvalue decomposition is extremely computationally demanding. As a contribution, in our paper, we seek an alternative that is more applicable in practice. In light of the non-convexity of (\ref{eq:m}), where optimizing $\mH$ is challenging due to its high order, we turn to optimize:
%Before we begin our formal result, we start with a toy case: $f(h,p)=(t-hp)^2+\lambda(h-p)^2$. By taking its partial derivative we have: $\partial_h f(h,p)=2p(hp-t)+2\lambda(h-p), \partial_p f(h,p)=2h(hp-t)+2\lambda(p-h)$, and therefore any critical point of $f$ satisfies $(p-h)(hp-t-2\lambda)=0$. One can see that as long as $\lambda>\frac{hp-t}{2}$, then $h^*=p^*$. The above example hints that instead of solving (\ref{eq:m}) directly, we may turn to solve:
\begin{equation}
	\label{eq:split}
	\min_{\mH,\mP} \|\mM- \mH\mP^T\|^2_F+\lambda\|\mH-\mP\|^2_F,
\end{equation}
where as long as $\lambda$ is sufficiently large, hopefully, we have $\mH^*=\mP^*$, which is not difficult to obtain  optimal solutions by utilizing any practical method such as alternating minimization.  The following theorem will give a more specific bound:

\begin{thm}
	Let ($\mH^*, \mP^*$) be critical points of (\ref{eq:split}), where $\sigma_n(\cdot)$ denotes the $n$-th largest eigenvalue and $\lambda>\frac{\|\mH^*\mP^{*T}\|_F-\sigma_n(\mM)}{2}$, then $\mH^* = \mP^*$ and $\mH^*$ is a critical point of original problem (\ref{eq:m}).
	\label{thm:hals}
\end{thm}

\begin{proof}[Proof of Theorem \ref{thm:hals}] 
	We first introduce the following lemma which is very useful for later proof~\cite{zhu2018dropping,li2021provable}.
	\begin{lem}\label{lem:ineqPSD}
		For any symmetric $\mA\in\R^{n\times n}$ and positive semi-definite matrix $\mB\in\R^{n\times n}$, we have:
		\[
		\sigma_n(\mA)\trace(\mB)\leq \trace\left(\mA\mB\right) \leq \sigma_1(\mA)\trace(\mB),
		\]
		where $\sigma_i(\mA)$ is the $i$-th largest eigenvalue of $\mA$.%\footnote{Due to space limit, we leave the proof to the supplemental.}
	\end{lem}
Now we turn to check the sub-differential of $f$ at critical point $(\mH^*,\mP^*)$ which is:
\e\begin{split} \label{eq:subdifferential}
	&\partial_{\mH^*}f(\mH^*,\mP^*) = 2[(\mH^*\mP^{*\T} - \mM)\mP^* + \lambda (\mH^* - \mP^*)]=0,\\
	&\partial_{\mP^*}f(\mH^*,\mP^*) = 2[(\mP^*\mH^{*\T} - \mM)\mH^* - \lambda (\mH^* - \mP^*)]=0.
\end{split}\ee
By subtracting the second line from the first, we have:
\e
	\label{keyeq}
	(2\lambda\mI+\mM)(\mH^*-\mP^*)=\mP^*\mH^{*\T}\mH^*-\mH^*\mP^{*\T}\mP^*.
\ee
By taking the inner product $\mH^\star - \mP^\star$ on both sides:
\e
\begin{aligned}
	&\langle 2\lambda \mId + \mM,(\mH^*-\mP^*)(\mH^*-\mP^*)^\T\rangle \\= &\langle  \mP^*\mH^{*\T}\mH^*-\mH^*\mP^{*\T}\mP^*, \mH^*-\mP^* \rangle.
%	content...
\end{aligned}
\label{eq: U V inner product}\ee
Applying Lemma \ref{lem:ineqPSD} on the LHS, we have:
\begin{equation}\label{left}
	%\begin{split}
	\langle 2\lambda \mId + \mM,(\mH^*-\mP^*)(\mH^*-\mP^*)^\T\rangle \\ \ge (2\lambda+\sigma_n(\mM))\|\mH^*-\mP^*\|^2_F,
		%\end{split}
\end{equation}
while applying Lemma \ref{lem:ineqPSD} on the other side we have:
\e\begin{aligned}
\label{right}
&\langle  \mP^*\mH^{*\T}\mH^*-\mH^*\mP^{*\T}\mP^*, \mH^*-\mP^* \rangle\\
=& \left\langle \frac{\mP^\star\mH^{\star\T} + \mH^\star\mP^{\star\T}}{2}, (\mH^\star - \mP^\star)(\mH^\star - \mP^\star)^\T \right\rangle \\&- \frac{\left\| \mH^\star\mP^{\star\T}- \mP^\star\mH^{\star\T}\right\|_F^2}{2}\\
 \leq &\left\langle \frac{\mP^\star\mH^{\star\T} + \mH^\star\mP^{\star\T}}{2}, (\mH^\star - \mP^\star)(\mH^\star - \mP^\star)^\T \right\rangle\\
\leq &\sigma_1\left( \frac{\mP^\star\mH^{\star\T} + \mH^\star\mP^{\star\T}}{2} \right)\|\mH^\star - \mP^\star\|_F^2\\
\leq &\| \frac{\mP^\star\mH^{\star\T} + \mH^\star\mP^{\star\T}}{2} \|_F\|\mH^\star - \mP^\star\|_F^2\\
\leq &\|\mH^*\mP^{*\T}\|_F\|\mH^\star - \mP^\star\|_F^2.
\end{aligned}
\ee
Combining the above two equations we have:
\e\label{ineq}
(2\lambda+\sigma_n(\mM))\|\mH^*-\mP^*\|^2_F\le \|\mH^*\mP^{*\T}\|_F\|\mH^\star - \mP^\star\|_F^2. 
\ee
Thus, if $\lambda>\frac{\|\mH^*\mP^{*T}\|_F-\sigma_n(\mM)}{2}$, then $\mH^*=\mP^*$ and any critical points satisfying (\ref{eq:subdifferential}) are also those for (\ref{eq:m}).
\end{proof}
\noindent The following lemma gives a bound for $\|\mH^*\mP^{*T}\|_F$.
\begin{lem} For (\ref{eq:split}), suppose the objective decreases with initialization $\mP_0=\mH_0$, then for any $k\geq 0$,  the iterate $(\mH_k,\mP_k)$ generated by any algorithm satisfies:
	\e%\begin{split}
		%\|\mH_k\|_F^2+\|\mP_k\|_F^2& \leq %\left(\frac{1}{\lambda}+2\sqrt{r}\right)\|\mX-\mH_0\mH_0^\T\|_F^2+2\sqrt{r} \|\mX\|_F:=B_0,\\
		\|\mH_k\mP_k^\T\|_F\leq\|\mM-\mH_0\mP_0^\T\|_F+\|\mM\|_F.
	%\end{split}	
\label{eqn:bound}\ee	\label{lem:bound:iterat}
\end{lem} 
\begin{proof}[Proof of Lemma \ref{lem:bound:iterat}]
	By the assumption that the algorithm decreases the objective function, we have:
	\begin{align*}
		\|\mM - \mH_k \mP_k^{\T } \|_F^2 + \lambda \|\mH_k-\mP_k \|_F^2
		&\leq  \|\mM - \mH_0\mH_0^{\T} \|_F^2\\
		\implies	\|\mM - \mH_k \mP_k^{\T } \|_F^2&\leq  \|\mM - \mH_0\mH_0^{\T} \|_F^2\\
		\implies \| \mH_k \mP_k^{\T }\|_F\le \|\mM\|_F+&\|\mM - \mH_k \mP_k^{\T } \|_F\\
		\le \|\mM\|_F+&\|\mM - \mH_0 \mP_0^{\T } \|_F.
	\end{align*}
	%which further implies that
	\label{lem:bound:iterate}
\end{proof}
\begin{cor}\label{col:lambda}
	If $\lambda>\frac{1}{2}(\|\mM\|_F+\|\mM-\mH_0\mP_0^{T}\|_F-\sigma_n(\mM))$ and $\mP_0=\mH_0$, then any algorithm decreases the objective in (\ref{eq:split}) will result in $\mP^*=\mH^*$.
\end{cor}
\begin{proof}
	This is established  by Lemma~\ref{lem:bound:iterat} and Theorem~\ref{thm:hals}.
\end{proof}

\begin{dis}
	In our case, $\mM$ is symmetric, not necessarily positive semi-definite (PSD). To compute $\sigma_n(\mM)$, which can be negative, it is very time-consuming if eigenvalue decomposition is utilized given its complexity level being $\mathcal{O}(n^3)$ when $n$ is large. Therefore, we can divide it into the following cases:
	\begin{itemize}
		\item when $\mM$ is PSD, then $\lambda>\frac{1}{2}(\|\mM\|_F+\|\mM-\mH_0\mP_0^{T}\|_F)$ will naturally satisfy the requirement, admitting desired optimal solutions.
		\item when $\mM$ is not PSD, we can first compute the leading eigenvalue ($t$) by power iteration or Lanczos method~\cite{kuczynski1992estimating}, which gives $\mathcal{O}(log(n)/k)$ and $\mathcal{O}((log(n)/k)^2)$ convergence rate, respectively~\cite{wright2022high}. By setting $\lambda>\frac{1}{2}(\|\mM\|_F+\|\mM-\mH_0\mP_0^{T}\|_F+|t|)$, one can verify that it will obtain desired solution with significantly reduced complexity than eigenvalue decomposition.
	\end{itemize}
	\end{dis}

To end this proof part, we propose an efficient algorithm that will \textbf{decrease the objective monotonically}. %By making use of the machinery developed above, we are closing the gap from hypothesis to conclusion  in Theorem \ref{thm:lambda}.

By noticing $\mH\mP^T=\vh_i\vp_i^T+\sum_{j\neq i}\vh_j\vp_j^T$, we can optimize $\mH$ and $\mP$  column by column. 
Different from the column-wise update process proposed in other studies~\cite{shi2017inexact, vandaele2016efficient}, there is no assumption needed in our update algorithm, the update formula in our method is much more straightforward and a detailed proof of sufficient decrease in objective function with our method is provided. 

Denote $\overline{\mM}=\mM-\sum_{j\neq i}\vh_j\vp_j^T$, we have:
\begin{equation}
	\begin{aligned}
	\vh_i^+=&\argmin_{\vh_i} \|\overline{\mM} - \vh_i\vp_i^\T\|^2 + \lambda\|\vh_i - \vp_i\|_2^2\\
	=&\argmin_{\vh_i} (\|\vp_i\|^2+\lambda)\vh_i^2-2\langle \vh_i,\lambda\vp_i+\overline{\mM}\vp_i\rangle\\
	=&\frac{(\overline{\mM}+\lambda\mI)\vp_i}{\|\vp_i\|^2+\lambda}.
\end{aligned}
\end{equation}

Moreover, by noticing the strongly convexity ($\alpha=2(\|\vp_i\|^2+\lambda)$) of objective \wrt \  $\vh$, one have:
\begin{equation}
	\label{sufficient_decrease_h}\begin{aligned}
	&f(\vh_i,\vp_i)-f(\vh_i^+,\vp_i)\\\ge &\langle \nabla_{\vh_i} f(\vh_i^+,\vp_i),\vh_i-\vh_i^+\rangle+ \frac{\alpha}{2}\|\vh_i-\vh_i^+\|^2\\
	\ge &\lambda\|\vh_i-\vh_i^+\|^2,
\end{aligned}
\end{equation}
%\end{lem} 
which indicates a sufficient decrease by updating $\vh_i$. Similarly, one can have the same conclusion while updating $\vp_i$. 

\begin{algorithm}[tb]
	\caption{Efficient update to optimize (\ref{eq:split})}
	\label{alg:alg2}
	\begin{algorithmic}
		%\State {\bfseries Input:} similarity matrix $\mA\in\real^{n\times n}$, number of clusters $k$, regularization parameters $\lambda$ and $\theta$, the Laplacian matrix $\mL$, number of iterations $K$
		\State {\bfseries Initialization:} $\mH_0=\mP_0\in\real^{n\times k}$.%, calculate $\mM$ = $\mA - \lambda \mL$, $i$ = 0
		\While{ not converge}
		\For {$i = 1:k$}
		\State %calculate $\mM-\mH\mH^T$ excluding $\vu_c\vv_c^T$: 
		$\overline{\mM} = \mM - \sum_{j \neq i}\vh_j\vp_j^T$.
		\State $\vh_i^{+} = \frac{(\overline{\mM}+\lambda\mI)\vp_i}{\|\vp_i\|^2 +\lambda}, \ \ \vp_i^{+} = \frac{(\overline{\mM}+\lambda\mI)\vh_i^+}{\|\vh^+_i\|^2 +\lambda}$.
		\EndFor
		%\State $i = i+1$
		\EndWhile
		\State {\bfseries Output:} $\mH^* = \mP^*$.
	\end{algorithmic}
\end{algorithm}

Fig.~\ref{fig:hals} shows the convergence curves of Algorithm\ref{alg:alg2}, SymANLS ~\cite{zhu2018dropping} and alternating direction method of multipliers (ADMM)~\cite{boyd2011distributed} when solving symmetric matrix factorization problem. From the figure, we see Algorithm \ref{alg:alg2} converges very fast which indicates its superiority. 
\begin{figure}
	\centering
	\includegraphics[width=0.8\linewidth]{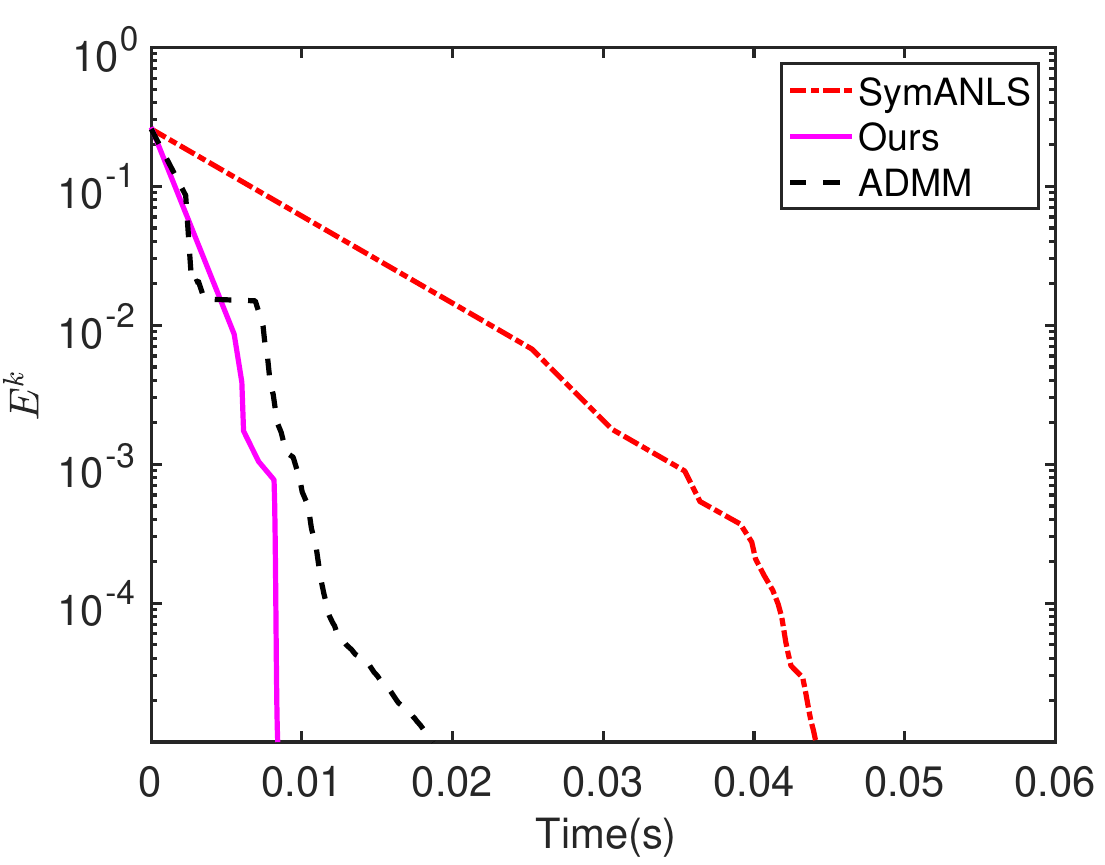}
	\caption{Typical convergence curve which shows the superiority of our proposed method.}
	\label{fig:hals}
\end{figure}

\subsection{A More General Framework for Constrained Optimization}
Though the above subsection describes a simple, yet very efficient algorithm to obtain an optimal solution, still it can't deal with most constraint problems such as $\|\vh\|_2=1, \|\vh\|_0\le s$, \etc~\cite{liu2021spherical}, where different constraint indicate various meanings such as sparsity, probability distribution \etc \ In this subsection we propose a more general framework to solve (\ref{eq:m}).% with gradient descent method utilizing the following altered second-order Taylor expansion with quadratic approximation:

Recall the relationship between second-order Taylor expansion and gradient descent:
\begin{equation}
	\label{eq:expan}
	f(\vy) \approx f(\vx) + \langle \nabla f(\vx), \vy-\vx\rangle + \frac{1}{2} \langle \vy-\vx,\nabla^2 f(\vx)(\vy-\vx)\rangle.
\end{equation}
If we replace Hessian $\nabla^2 f(\vx)$ with $\frac{1}{t}\mI$, then $\min f(\vy)$ is to minimize $ \|\vy-\vx+t\nabla f(\vx)\|^2_2$, where $t$ is the step size in gradient descent method. It can be verified that:
\begin{equation}\label{eq:exph}
%\begin{aligned}
%f(\mH) &= \|\mM- \mH\mH^T\|^2_F \\
%&= tr(\mM- \mH\mH^T)(\mM- \mH\mH^T)^T \\
%&= tr(\mH\mH^T\mH\mH^T - 2\mH\mH^T\mM+\mM\mM),\\
%\end{aligned}
%\end{equation}
%\begin{equation}\label{eq:exphd}
%\begin{aligned}
\nabla f(\mH)= 4(\mH\mH^T\mH - \mM\mH).
%\end{aligned}
\end{equation}
%In each iteration we aim to update $\mH^+$ to minimize $f(\mH^+)$.
By invoking (\ref{eq:expan}) to update $\mH^+$, we get:
\begin{equation}\label{eq:h}
\begin{aligned}
\mH^+ &= \argmin_{\mH'\in\setH} f(\mH') \\
&= \argmin_{\mH'\in\setH} f(\mH)  + \langle\nabla f(\mH), \mH'-\mH\rangle  + \frac{1}{2t} \|\mH'-\mH\|_F^2\\
&= \argmin_{\mH'\in\setH}  \frac{1}{2t} \|\mH' - (\mH - t \nabla f(\mH))\|_F^2\\
&= \mathcal{P}_{\setH}(\mH- t \nabla f(\mH)).
\end{aligned}
\end{equation}
$\setH$ denotes the feasible set satisfying the constraint. Stepsize $t$ in the above update rule should be  relatively small  to avoid gradient explosion. However, if it is too small, the convergence becomes slow, which should be avoided as well. As a contribution, we propose a method with \textbf{adaptive stepsize} which will make the objective decrease monotonically.
\begin{algorithm}[tb]
	\caption{Optimize (\ref{eq:m}) where $\mH\in\setH$}
	\label{alg:alg}
	\begin{algorithmic}
		\State {\bfseries Input:} $\mM$ = $\mA - \lambda \mL$.%similarity matrix $\mA\in\real^{n\times n}$, number of clusters $k$, regularization parameters $\lambda$, number of iterations $K$, step size $t$
		\State {\bfseries Initialization:} $\mH_0\in\setH$, $i$ = 0.
		\While{$i < K$}
		\State $\nabla f(\mH_i) = 4(\mH_i \mH_i^T \mH_i - \mM\mH_i)$.
		\State $L_i = 4\sigma_{max}(\mH_i\mH_i^T - \mM) + 8 \sigma_{max}(\mH_i^T\mH_i)$.
		\State $\mH_{i+1} = \mathcal{P}_\setH(\mH_i - t \nabla f(\mH_i))$, where $t=\frac{1}{2L_i}$.
		\State $i = i+1$.
		\EndWhile
		\State {\bfseries Output:} $\mH_K$.
	\end{algorithmic}
\end{algorithm}

Algorithm \ref{alg:alg} provides a generalized framework to solve any symmetric matrix factorization with different constraints. %problems, it doesn't impose any constraints on $\mA$ and $\mH$. It is natural to extend the update rule to problems with constraints on $\mH$, 
Below we provide some concrete examples:
\begin{itemize}
	\item \emph{Example I}: Nonnegative constraint $\mH \geq 0$\footnote{One can verify that Algorithm \ref{alg:alg2} can also work if we set every column nonnegative, which can simply be obtained via $\vh_i^{+} = \max\{\frac{(\overline{\mM}+\lambda\mI)\vp_i}{\|\vp_i\|^2+\lambda},0\}$, $\vp_i^{+} =\max\{ \frac{(\overline{\mM}+\lambda\mI)\vh_i^+}{\|\vh^+_i\|^2+\lambda},0\}$. }:
	\begin{equation}\label{eq:hn}
		\begin{aligned}
			\mH^+ = \max\{\mH- t \nabla f(\mH), 0\}.
		\end{aligned}
	\end{equation}
	
	\item \emph{Example II}: Unit constraint $\|\vh\|_2 = 1$:
	\begin{equation}\label{eq:unit}
		\begin{aligned}
			\vh^+=\frac{\vh-t\nabla f(\vh)}{\|\vh-t\nabla f(\vh)\|_2}.
		\end{aligned}
	\end{equation}
	
	\item \emph{Example III}: Sparsity constraint $\|\vh\|_0 \le s$:
	
	WLOG, assume the top $s$ entry with maximum magnitude in $\vh$ is indexed as $[1,s]$, then~\cite{liu2019spherical}
	\begin{equation}\label{eq:hs}
		\begin{aligned}
			\vh^+ = 
			\begin{cases}
				(\vh- t \nabla f(\vh))_{j} & \text{if } j \in [1,s],\\  
				0  & \text{otherwise}.
			\end{cases}
		\end{aligned}
	\end{equation}
	
	\item \emph{Example IV}: Orthogonality Constraint $\mH^T\mH = \mI$:
	\begin{equation}\label{eq:ortho}
		\begin{aligned}
			\mH^+&=\mU^T\mV,\\where \  [\mU,\mSigma,\mV]&=svd(\mH- t \nabla f(\mH)).
		\end{aligned}
	\end{equation}

	\item \emph{Example V}: $\ell_1$-norm constraint on $\|\vh\|_1\le\alpha$~\cite{beck2017first}.
	By denoting $\mathcal{T}_\lambda(\vh)=[\vh-\lambda\ve]_+\odot sgn(\vh)$, we have:
	\begin{equation}\label{eq:l1}
		\begin{aligned}
			\vh^+ = 
			\begin{cases}
				\vh& \text{if } \|\vh\|_1\le \alpha,\\  
				\mathcal{T}_{\lambda^*}(\vh)  & \text{otherwise},
			\end{cases}
		\end{aligned}
	\end{equation}
	where $\lambda^*$ is positive root for $\|\mathcal{T}_\lambda(\vh)\|_1=\alpha$ which can be solved within $\mathcal{O}(nlog(n))$~\cite{duchi2008efficient}.
\end{itemize}
%Moreover, when $\mH$ has nonnegative constraint as Example I, we can utilize the fast update algorithm in Algorithm~\ref{alg:alg2} \cite{zhu2018dropping} to solve Eq.~(\ref{eq:sim}) by optimizing over column by column, we use subscript $i$ to denote the $i$-th iteration.
%
%Algorithm~\ref{alg:alg2} aims to solve following regularized form:
%\begin{equation}\label{eq:hals}
%\begin{aligned}
%\argmin_{\mH \geq 0, \mP \geq 0}f(\mH, \mP) = \|\mA - \mH\mP^T\|^2_F + \theta\|\mH - \mP\|^2_F 
%\end{aligned}
%\end{equation}
%where $\mA$ is a symmetric matrix, $\theta$ is the regularization parameter. Theorem~\ref{thm:hals} has been proved in \cite{zhu2018dropping} so we can adopt methods for nonsymmetric NMF while still can get a solution for symmetric NMF.

\section{Convergence Analysis}\label{sec:convergence}
%In this section we will show the convergence of our proposed algorithm. 
In the last section we mentioned the step size $t$ should not be too large or small, and in this section, we will determine the best $t$ in each update which guarantees the objective decreases monotonically by introducing the following lemma to begin:
%\begin{defi}
%	A differentiable function $f$ is said to have an $L$-Lipschitz continuous gradient if for some $L>0$
%\e
%\lVert \nabla f(x) - \nabla f(y)\rVert \le L \lVert x-y\rVert,~\forall x,y.
%\label{eq:Lipschitz h}\ee
%\end{defi}
%\begin{remark}
%	The definition does not assume convexity of $f$.
%\end{remark}
%Also we provide a useful lemma~\cite{zhou2018fenchel}:
\begin{lem}
For a function $f$ with a Lipschitz continuous gradient $L$, if
$
\|\nabla f(\vx) - \nabla f(\vy)\|_2 \leq L\|\vx - \vy\|_2
$ then 
$
f(\vy) \leq f(\vx) + \langle\nabla f(\vx), \vy-\vx\rangle + \frac{L}{2}\|\vy-\vx\|_2^2
$.
\label{lemma:L}
\end{lem}
\begin{prop}\label{prop:Lipschitz h}
For (\ref{eq:m}), in each $i$-th update, $L_i = 4\sigma_{max}(\mH_i\mH_i^T - \mM) + 8 \sigma_{max}(\mH_i^T\mH_i)$. 
\end{prop}
\begin{proof}[Proof of Proposition~\ref{prop:Lipschitz h}]
For sake of simplicity, we denote $\mH$ as $\mH_i$. With Lemma~\ref{lemma:L}, it is equivalent to show $f(\mH') - f(\mH) - \langle \nabla f(\mH), \mH'-\mH\rangle \leq \frac{L}{2}\|\mH'-\mH\|_F^2$. By denoting $\mH'$ as $\mH + \Delta \mH$, we have:
\begin{equation}\small
\label{eq:hessian}
\begin{split}
& \ \ \ f(\mH+\Delta \mH) - f(\mH) - \langle\Delta \mH, \nabla f(\mH)\rangle \\
%&= \|\mM - (\mH+\Delta \mH)(\mH+\Delta \mH)^T\|_F^2 \\
  %&\ \ \ - \|\mM - \mH \mH^T\|_F^2 - \langle\Delta \mH, \nabla f(\mH)\rangle\\
&= 2 \langle \mH\mH^T - \mM, \Delta \mH\Delta \mH^T\rangle + \|\mH\Delta \mH^T + \Delta \mH \mH^T\|_F^2 \\
& \leq 2(\langle \mH\mH^T - \mM, \Delta \mH\Delta \mH^T\rangle + \|\mH\Delta \mH^T\|_F^2 +\|\Delta \mH \mH^T\|_F^2 )\\
&= 2tr(\Delta \mH^T(\mH\mH^T - \mM)\Delta \mH) + 4tr(\Delta \mH \mH^T \mH\Delta \mH^T)\\
& \leq \lbrack 2\sigma_{max}(\mH\mH^T - \mM) + 4 \sigma_{max}(\mH^T\mH)\rbrack \|\Delta \mH\|_F^2=\frac{L}{2} \|\Delta \mH\|_F^2,
\end{split}
\end{equation}
where $\sigma_{max}(\cdot)$ denotes the maximum singular value. 
%Thus:
%\begin{equation}
%\label{eq:hessianresult}
%%\begin{split}
%f(\mH+\Delta \mH) - f(\mH) - \langle\Delta \mH, \nabla f(\mH)\rangle \leq \frac{L}{2} \|\Delta \mH\|_F^2.
%%\end{split}
%\end{equation}
%According to Lemma ~\ref{lemma:L}, there is a Lipschitz continuous gradient $L$ for $f(\mH)$ with $L = 4\sigma_{max}(\mH\mH^T - \mM) + 8 \sigma_{max}(\mH^T\mH)$.
%Moreover, if we assume the objective decreases monotonically, then we have:%we prove $L_i$ is bounded, so the step size in gradient descent method can be sufficient large to get to the critical point:
%\begin{equation}
%\label{eq:boundL}
%\begin{split}
%L &= 4\sigma_{max}(\mH\mH^T - \mM) + 8 \sigma_{max}(\mH^T\mH) \\
%& \leq 4\|\mH\mH^T - \mM\|_F + 8\|\mH^T\mH\|_F\\
%& \leq 4\|\mH\mH^T - \mM\|_F + 8(\|\mM\|_F + \|\mH^T\mH-\mM\|_F)\\
%%&= 12\|\mH^T\mH-\mM\|_F+8\|\mM\|_F\\
%& \leq 12\|\mH_0^T\mH_0-\mM\|_F+8\|\mM\|_F,
%\end{split}
%\end{equation}
%where $\mH_0$ is the initialization of $\mH$ in Algorithm \ref{alg:alg}.
\end{proof}

\begin{figure*}[t]
	\centering
	\includegraphics[width=.33\linewidth, height=4.8cm]{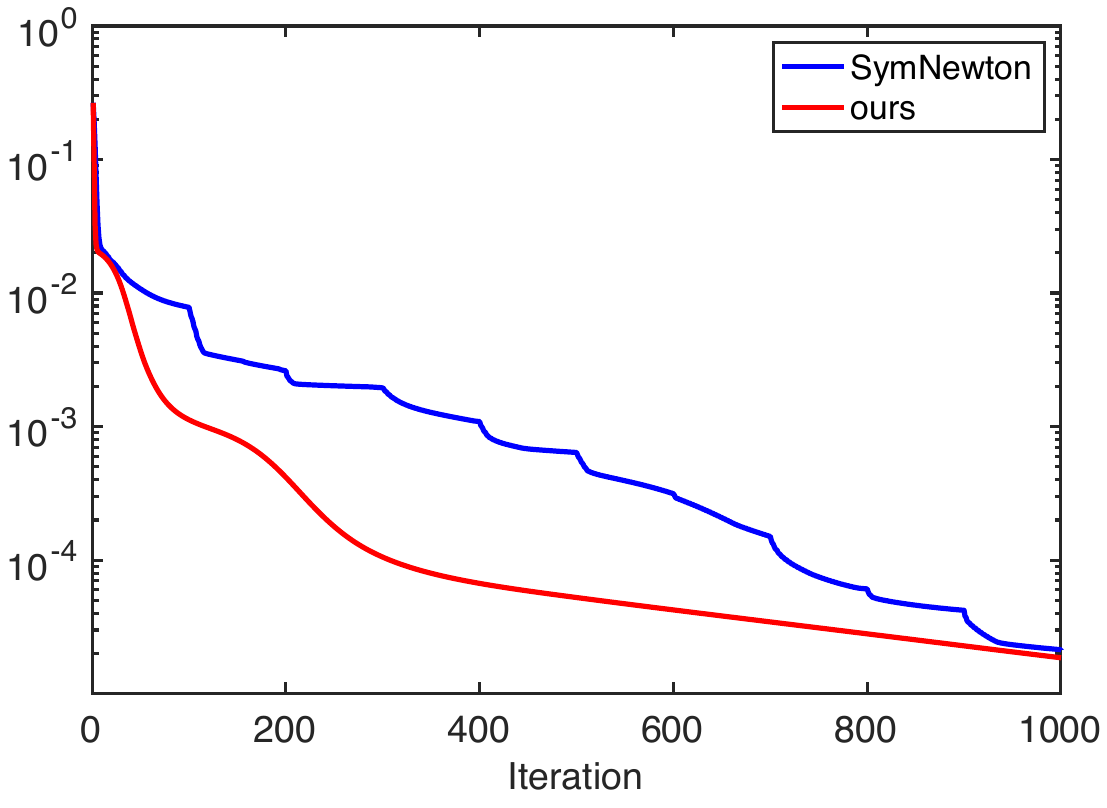}\hfill
	\includegraphics[width=.33\linewidth, height=4.8cm]{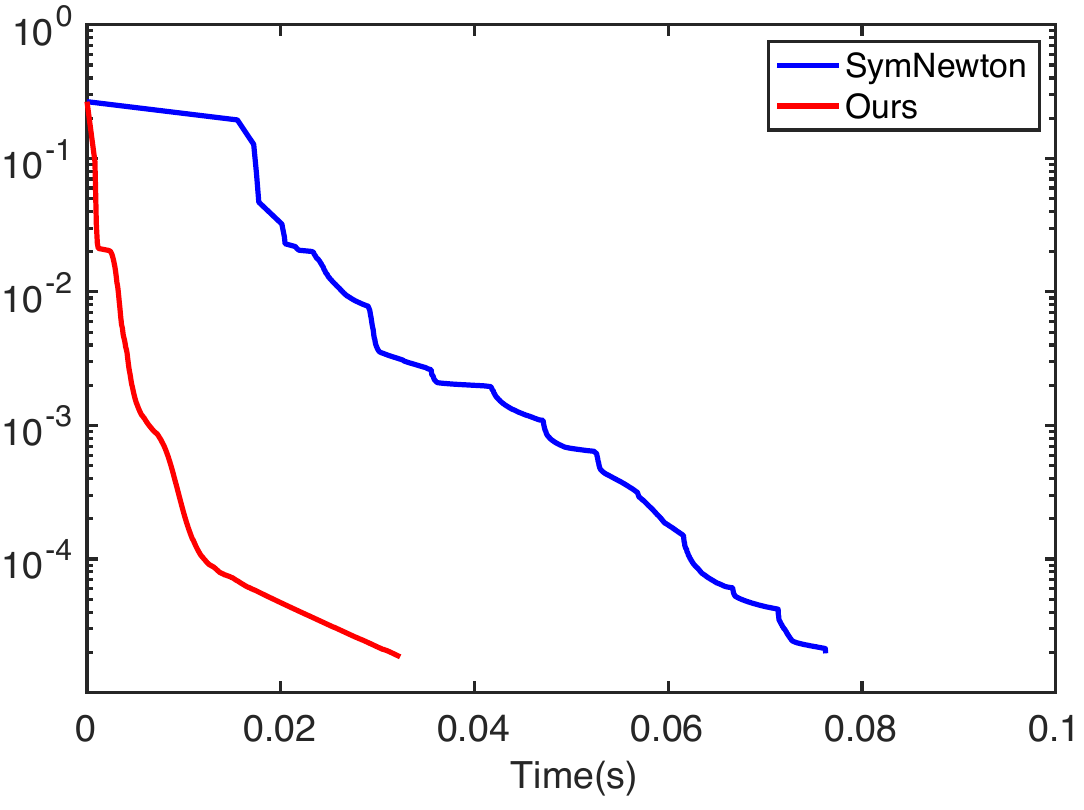}\hfill
    \includegraphics[width=.33\linewidth, height=4.8cm]{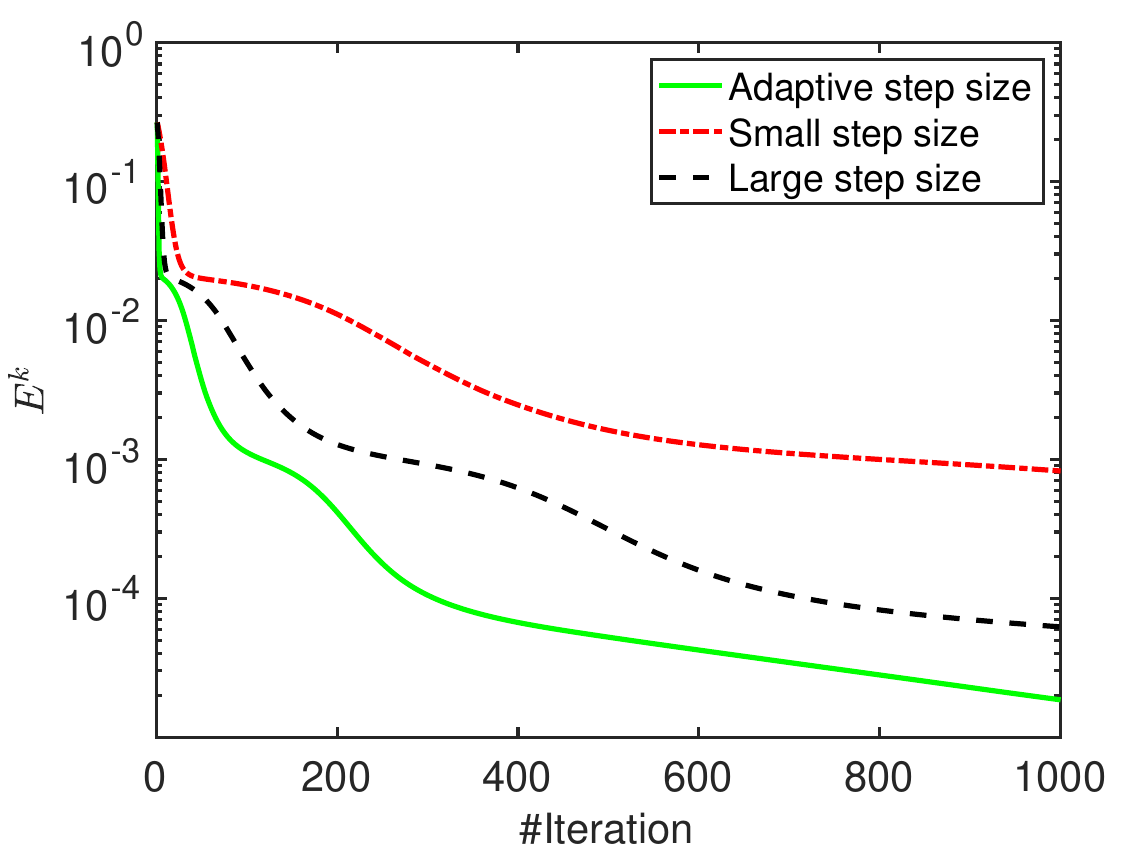}
	\caption{Convergence comparison \wrt \ $E^k := \frac{\|\mM-\mH^k(\mH^{k})^T\|_F^2}{\|\mM\|_F^2}$ . \textbf{Left:} objective value versus iteration numbers. \textbf{Middle:} objective value versus time. \textbf{Right:} projected gradient descent with different step-size settings.}
	\label{fig:convergerate}
\end{figure*}
\begin{figure*}[h!]
	\centering
	\includegraphics[width=.33\linewidth]{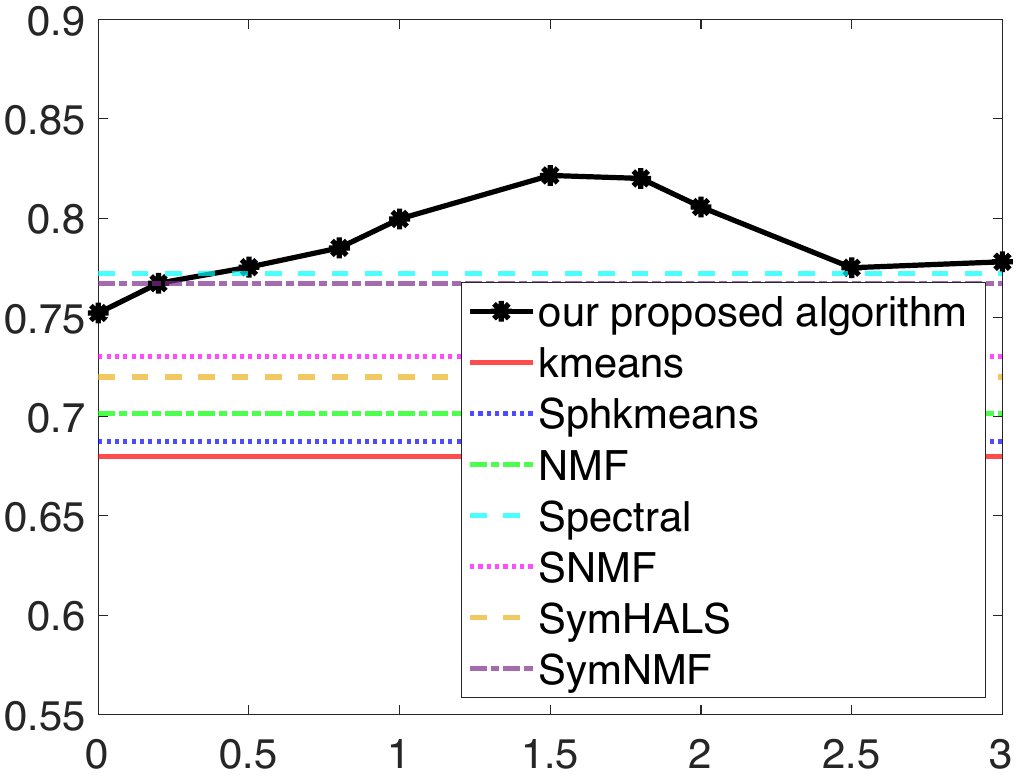}\hfill
	\includegraphics[width=.33\linewidth]{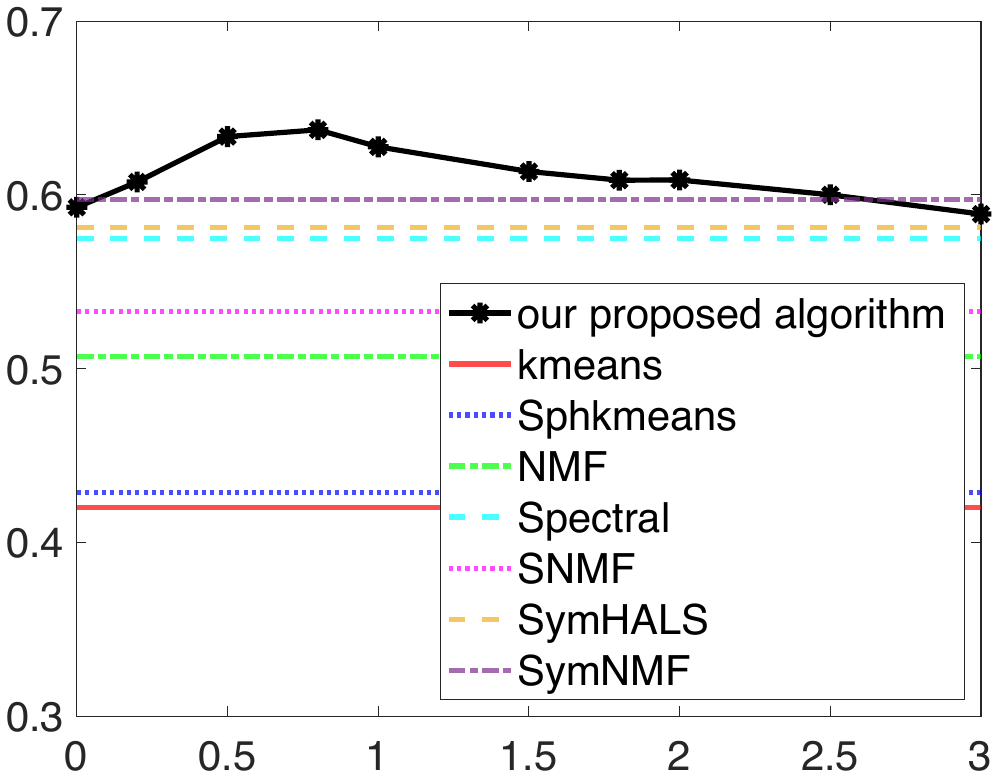}\hfill
	\includegraphics[width=.33\linewidth]{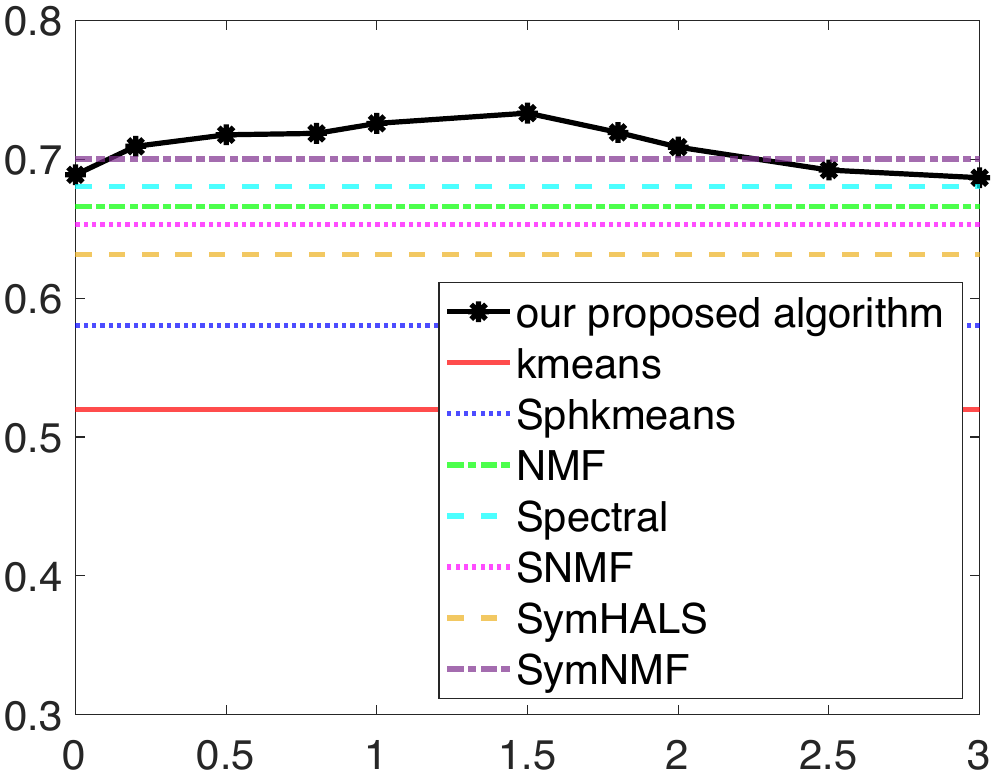}
	\caption{Accuracy (Y-axis) comparison with $\lambda$ changes (X-axis). \textbf{Left:} on a subset of 5 clusters from COIL-20. \textbf{Middle:} on a subset of 12 clusters from Reuters-21578. \textbf{Right:} on a subset of 10 clusters from TDT2.}
	\label{fig:lambda}
\end{figure*}
Below we show that the objective function $f(\mH)$ in Algorithm~\ref{alg:alg} has sufficient decreasement in each update with step size $t = \frac{1}{2L}$ and the generated sequence is convergent. 

\begin{thm}
Let $g(\mH_k):=f(\mH_k)+\mathcal{C}(\mH_k)$ be the objective function sequence generated by Algorithm~\ref{alg:alg} with constant step size $t_k = \frac{1}{2L_k}$. Then the sequence $g(\mH_k)$ obeys sufficient decrease:
%\begin{equation}\label{eq:sufficient decrease}
$g(\mH_{k-1}) - g(\mH_{k})\geq \frac{L_k}{2}\| \mH_{k}-\mH_{k-1}\|_F^2$.
%\end{equation}
\label{thm:convergence}
\end{thm}

\begin{proof}[Proof of Theorem~\ref{thm:convergence}]
    From (\ref{eq:h}), noting that $\mH_k$ minimizes $J(\mH) =\langle  \nabla f(\mH_{k-1}), \mH-\mH_{k-1}\rangle + \frac{1}{2t} \|\mH-\mH_{k-1}\|_F^2+\mathcal{C}(\mH)$, where $\mathcal{C}(\mH)$ represents any constraint $\mathcal{C}(\mH)=
    \begin{cases}
    	0, \mH\in\setH\\
    	\infty, \text{ else} 
    \end{cases}$, thus we naturally have $J(\mH_k) \leq J(\mH_{k-1})$, which implies:
    \begin{equation}\begin{split}
    %\footnotesize
        & \langle  \nabla f(\mH_{k-1}), \mH_k-\mH_{k-1}\rangle  + \frac{1}{2t} \|\mH_k-\mH_{k-1}\|_F^2 \\ \leq &\mathcal{C}(\mH_{k-1})-\mathcal{C}(\mH_k).
    \label{eq:sufi}
    	\end{split}
    \end{equation}
    According to the definition of Lipschitz continuous, Lemma.~\ref{lemma:L}, and when $t = \frac{1}{2L_k}$ in (\ref{eq:sufi}), we have:
    \begin{equation}
    \begin{aligned}
        & f(\mH_k)+\mathcal{C}(\mH_k)-f(\mH_{k-1})-\mathcal{C}(\mH_{k-1}) \\
         \leq &\langle  \nabla f(\mH_{k-1}), \mH_k-\mH_{k-1}\rangle  + \frac{L_k}{2}\|\mH_k-\mH_{k-1}\|_F^2\\
         &+\mathcal{C}(\mH_k)-\mathcal{C}(\mH_{k-1})\\
         \leq &-L_k \|\mH_k-\mH_{k-1}\|_F^2+ \frac{L_k}{2}\|\mH_k-\mH_{k-1}\|_F^2\\
         = &-\frac{L_k}{2}\|\mH_k-\mH_{k-1}\|_F^2.
    \end{aligned}
    \label{eq:conv}
    \end{equation}
%    which gives Eq.~(\ref{eq:sufficient decrease}).
%    Now repeating Eq. (\ref{eq:sufficient decrease}) for all $k$ will give:
%	\e\label{eq:H_update}
%	\frac{\bar{L}}{2}\sum_{k=1}^\infty\|\mH_{k} - \mH_{k-1}\|_F^2 \leq g(\mH_0),
%	\ee
%	where $\bar{L}:=\min\{L_1,\dots,L_k\}>0$ and gives Eq.~(\ref{eq:diff goes to 0}).
\end{proof}
One can see that if $\mH_k\in\setH$, then (\ref{eq:conv}) degenerates to $f(\mH_{k-1}) - f(\mH_{k})\geq \frac{L_k}{2}\| \mH_{k}-\mH_{k-1}\|_F^2$. Therefore, (\ref{eq:m}) decreases with update. Now repeating for all $k$ will give
	$\frac{\bar{L}}{2}\sum_{k=1}^\infty\|\mH_{k} - \mH_{k-1}\|_F^2 \leq g(\mH_0)$,
	where $\bar{L}:=\min\{L_1,\dots,L_k\}>0$ and establishes its convergence.
%With Eq.~(\ref{eq:sufficient decrease}), we have the following conclusion regarding  convergence rate of our proposed algorithm:
\begin{thm}
In Algorithm \ref{alg:alg}, to ensure $\min \|\nabla f(\mH)+\partial\mathcal{C}(\mH)\|^2_F \leq \epsilon$, we need at most $T = O(\frac{1}{\epsilon})$ iterations. 
\label{thm:rate}
\end{thm}

This indicates that the algorithm reaches a critical point at least a sub-linear convergence rate. Below is the proof:
\begin{proof}
    First, by definition: 
\e
\mH^+= \argmin_{\mH'} \langle \nabla f(\mH), \mH'-\mH\rangle + \frac{1}{2t} \|\mH'-\mH\|_F^2+\mathcal{C}(\mH'),
\ee

which implies:
\e\label{criticalpintcondition}
0\in \nabla f(\mH) +\frac{1}{t}(\mH^+-\mH)+\partial \mathcal{C}(\mH^+).
\ee
WLOG, now we define $\mA_{\mH^+}\in \nabla f(\mH^+)+\partial \mathcal{C}(\mH^+)$, apparently it indicates the gap to $\mH^*$ which satisfies $0 \in \nabla f(\mH^*)+\partial \mathcal{C}(\mH^*)$. Accordingly:
\e%\begin{split}
\small
	\mA_{\mH^+}\in \nabla f(\mH^+)+\partial \mathcal{C}(\mH^+)
	= \nabla f(\mH^+)- \nabla f(\mH)-\frac{1}{t}(\mH^+-\mH).
%\end{split}
\ee

As $f$ is $L$-Lipschitz gradient continuous, then: 

\e\begin{split}
	\|\mA_{\mH^+}\|&\leq \|\nabla f(\mH^+)- \nabla f(\mH)\|+\frac{1}{t}\|\mH^+-\mH\|\\
	&\leq L \|\mH^+-\mH\|+\frac{1}{t}\|\mH^+-\mH\|\\
	&=3L\|\mH^+-\mH\|,
\end{split}
\ee
where the first line comes from subadditivity inequality, the second line is by definition of Lipschitz gradient continuity, while the last line is by definition of $t=\frac{1}{2L}$.

On the other hand, $f(\mH)-f(\mH^+)\ge\frac{L}{2}\|\mH^+-\mH\|^2_F$, therefore:
\e\begin{split}
	\|\mA_{\mH^+}\|^2_F&\leq 9L^2\|\mH^+-\mH\|^2_F\\
	&\le 9L^2*\frac{2}{L}(f(\mH)-f(\mH^+))\\
	&=18L(f(\mH)-f(\mH^+)).
\end{split}
\ee
By repeating the above for all $k$:
\e
\sum_{i} \|\mA_i\|_F^2\le18L(f(\mH_0)-f(\mH_{k}))\le 18Lf(\mH_0).
\ee
Thus, $\min \|\mA_i\|_F^2\le \frac{18Lf(\mH_0)}{T}$, that is as long as $T=\frac{18Lf(\mH_0)}{\epsilon}=\mathcal{O}(\frac{1}{\epsilon})$, then $\min \|\mA_i\|_F^2=\epsilon$ which finishes the proof. Apparently, our algorithm has a at least sub-linear convergence rate.
\end{proof}
%\begin{proof}[Proof of Theorem~\ref{thm:rate}]
%Now repeating Eq. (\ref{eq:sufficient decrease}) for all $k$ from 1 to $T$, we have 
%    \e\label{eq:HT}
%	\sum_{k=1}^T\|\mH_{k} - \mH_{k-1}\|_F^2 \leq \frac{2}{\bar{L}}f(\mH_0),
%	\ee
%And it is obvious that 
%    \e
%	\sum_{k=1}^T\min_{j=1 \cdots T} \|\mH_{j} - \mH_{j-1}\|_F^2 \leq \sum_{k=1}^T\|\mH_{k} - \mH_{k-1}\|_F^2,
%	\ee
%So finally we get
%	$\min\|\Delta \mH\|^2_F \leq \frac{2f(\mH_0)}{L_iT}$.
%If we run it $T = O(\frac{1}{\epsilon})$ iterations, we will get at least one, the minimum $\|\Delta \mH \|^2_F \leq \epsilon$. 
%\end{proof}

We compare our proposed method with the nonnegative constraint on $\mH$ using projected gradient descent and the Newton-like method in SymNMF~\cite{kuang2012symmetric} in terms of convergence speed on synthetic data in Fig.~\ref{fig:convergerate}. When run with the adaptive step-size setting, it converges way faster than others.%Y-axis is the objective function value over $\|A\|_F^2$ in a base-10 logarithmic scale. Our proposed method has faster convergence speed than SymNMF Newton-like method, and it converges faster with adaptive step size as described in Algorithm \ref{alg:alg}.
\begin{table}[h!]
	\normalsize
\caption{Datasets Information}
	\begin{tabular}{*{4}{c}}
		\toprule
		\multicolumn{1}{c}{Dataset}  &
		\multicolumn{1}{c}{\#Clusters}  &
		\multicolumn{1}{c}{\#Samples} &
		\multicolumn{1}{c}{Dimensionality} \\
		\midrule
		COIL-20& 20 & 1440 & 1024\\
		CIFAR-10& 10 & 3000 & 1024\\
		Reuters-21578& 41 & 8213 & 18933\\
		TDT2& 30 & 9394 & 36771\\
		\bottomrule
	\end{tabular}
	
	\label{tab:data}
\end{table}
\section{Experiments}\label{sec:exp}
\label{exp}

\begin{table*}[t!]
\caption{Normalized mutual information (NMI) of different algorithms on four datasets with varying numbers of clusters}
	\small
	\begin{center}	
		\begin{tabular}{*{13}{c}}
			\toprule
			\multicolumn{1}{c}{}  &
			\multicolumn{3}{c}{COIL-20}  &
			\multicolumn{3}{c}{CIFAR-10} &
			\multicolumn{3}{c}{Reuters-21578} &
			\multicolumn{3}{c}{TDT2}  \\
			\cmidrule(r){1-1}\cmidrule(r){2-4}\cmidrule(r){5-7}\cmidrule(r){8-10}\cmidrule(r){11-13}
			{Method} & {2} & {10} & {20} & {3} & {6} & {10} & {2} & {8} & {15} & {2} & {10} & {20}\\
			\midrule
			$K$-means & 0.901 & 0.624 & 0.591 & 0.296 & 0.287 & 0.201 & 0.785 &0.553 &0.421 & 0.752 & 0.532 & 0.501\\
			NMF & 0.907 & 0.729 & 0.522 & 0.308 & 0.288 & 0.195 & 0.819 &0.752 &0.598 & 0.822 & 0.666 & 0.600\\
			Spectral & 0.877 & 0.701 & 0.677 & 0.309 & 0.298 & 0.201 & 0.828 &0.611 & 0.499 & 0.829 & 0.607 & 0.592\\
			SymHALS & 0.911 & 0.688 & 0.652 & 0.308 & 0.255 & 0.198 & 0.855 & 0.631 &0.552 & 0.822 & 0.611 & 0.588\\
			SNMF & 0.911 & 0.659 & 0.638 & 0.311 & 0.302 & 0.188 & 0.872 &0.566 &0.531 & 0.751 & 0.689 & 0.603\\
			SymNMF & 0.951 & 0.739 & 0.662 & 0.307 & 0.289 & 0.209 & 0.897 & 0.692 &0.605 & 0.802 & 0.671 & 0.662\\
			GNMF & 0.951 & 0.701 & 0.652 & 0.319 & 0.301 & 0.217 & 0.852 &0.772 &0.618 & 0.798 & 0.699 & 0.652\\
			DSC & 0.949 & 0.752 & 0.701 & 0.318 & 0.302 & 0.297 & 0.876 &0.779 &0.682 & 0.851 & 0.691 & 0.664\\
			\textbf{Ours} & \textbf{0.958} & \textbf{0.797} & \textbf{0.725} & \textbf{0.341} & \textbf{0.319} & \textbf{0.302} & \textbf{0.901} &\textbf{0.798} &\textbf{0.755} & \textbf{0.861} & \textbf{0.729} & \textbf{0.705}\\
			\bottomrule
		\end{tabular}
	\end{center}
	\label{tab:nmi}
\end{table*}

\begin{table*}[t!]
\caption{Clustering accuracy (AC) of different algorithms on four datasets with varying numbers of clusters}
	\small
	\begin{center}	
		\begin{tabular}{*{13}{c}}
			\toprule
			\multicolumn{1}{c}{}  &
			\multicolumn{3}{c}{COIL-20}  &
			\multicolumn{3}{c}{CIFAR-10} &
			\multicolumn{3}{c}{Reuters-21578} &
			\multicolumn{3}{c}{TDT2}  \\
			\cmidrule(r){1-1}\cmidrule(r){2-4}\cmidrule(r){5-7}\cmidrule(r){8-10}\cmidrule(r){11-13}
			{Method} & {2} & {10} & {20} & {3} & {6} & {10} & {2} & {8} & {15} & {2} & {10} & {20}\\
			\midrule
			$K$-means & 0.921 & 0.674 & 0.631 & 0.316 & 0.297 & 0.221 & 0.815 &0.563 &0.503 & 0.800 & 0.581 & 0.533\\
			NMF & 0.923 & 0.765 & 0.586 & 0.330 & 0.306 & 0.198 & 0.900 &0.777 &0.616 & 0.839 & 0.693 & 0.629\\
			Spectral & 0.898 & 0.737 & 0.702 & 0.332 & 0.316 & 0.228 & 0.889 &0.645 &0.513 & 0.896 & 0.639 & 0.602\\
			SymHALS & 0.927 & 0.703 & 0.682 & 0.321 & 0.287 & 0.219 & 0.881 &0.658 &0.575 & 0.851 & 0.639 & 0.601\\
			SNMF & 0.932 & 0.686 & 0.653 & 0.328 & 0.312 & 0.202 & 0.890 &0.586 &0.552 & 0.781 & 0.703 & 0.630\\
			SymNMF & 0.972 & 0.772 & 0.695 & 0.321 & 0.305 & 0.231 & 0.911 &0.721 &0.627 & 0.823 & 0.703 & 0.689\\
			GNMF & 0.968 & 0.722 & 0.683 & 0.343 & 0.326 & 0.232 & 0.871 &0.802 &0.628 & 0.825 & 0.715 & 0.676\\
			DSC & 0.972 & 0.788 & 0.722 & 0.339 & 0.321 & 0.318 & 0.903 &0.811 &0.703 & 0.863 & 0.720 & 0.687\\
			\textbf{Ours} & \textbf{0.975} & \textbf{0.815} & \textbf{0.787} & \textbf{0.369} & \textbf{0.346} & \textbf{0.328} & \textbf{0.922} &\textbf{0.823} &\textbf{0.782} & \textbf{0.878} & \textbf{0.756} & \textbf{0.736}\\
			\bottomrule
		\end{tabular}
	\end{center}
	
	\label{tab:result}
\end{table*}

\subsection{Datasets}

Two image datasets and two text datasets are used in the experiment: COIL-20~\cite{coil20}, CIFAR-10~\cite{krizhevsky2009learning}, Reuters-21578~\cite{r21}, and TDT2~\cite{tdt}. Detailed descriptions of the number of clusters, number of samples, and dimensionality of these datasets can be found in Table~\ref{tab:data}.

\subsection{Experimental Settings}
Clustering performances of the following 9 algorithms are compared:
\begin{enumerate}
    \item Standard $K$-means;
    \item NMF using alternating nonnegative least squares algorithm~\cite{kim2008nonnegative}; The data matrix $\mX$ is transformed into its normalized-cut weighted version;
    \item Spectral clustering (Spectral)~\cite{ng2002spectral,von2007tutorial};
    \item Hierarchical Alternating Least Squares (HALS) for symmetric NMF (SymHALS)~\cite{zhu2018dropping};
    \item Symmetric NMF using Procrustes rotations (SNMF)~\cite{huang2013non};
    \item Symmetric NMF (SymNMF)~\cite{kuang2012symmetric};
    \item Graph regularized nonnegative matrix factorization (GNMF)~\cite{cai2010graph};
    \item Deep subspace clustering (DSC)~\cite{ji2017deep};
    \item Our method. Algorithm~\ref{alg:alg} is used to solve the objective function (\ref{eq:m}).
\end{enumerate}

% (1) Standard $K$-means;
% (2) NMF using alternating nonnegative least squares algorithm~\cite{kim2008nonnegative}; The data matrix $\mX$ is transformed into its normalized-cut weighted version.
% (3) Spectral clustering (Spectral)~\cite{ng2002spectral,von2007tutorial};
% (4) Hierarchical Alternating Least Squares (HALS) for symmetric NMF (SymHALS)~\cite{zhu2018dropping};
% (5) Symmetric NMF using Procrustes rotations (SNMF)~\cite{huang2013non};
% (6) Symmetric NMF (SymNMF)~\cite{kuang2012symmetric}.;
% (7) Graph regularized nonnegative matrix factorization (GNMF)~\cite{cai2010graph};
% (8) Deep subspace clustering (DSC)~\cite{ji2017deep} and (9) ours. Algorithm~\ref{alg:alg} is used to solve the objective function Eq.~(\ref{eq:m}).

In order to randomize the experiments, we conduct the evaluation using subsets of the whole datasets with different cluster numbers. For each selected number of clusters $K$, 10 test runs are conducted on a randomly chosen subset with $K$ clusters. When $K$ is the total number of clusters in the complete data set, the test runs are repeated on the entire data set. The symmetric matrix $\mA$ can be obtained by utilizing any similarity measures, for simplicity we use the inner product similarity. Throughout the experiments, we use Matlab R2019a on a laptop with a 1.4 GHz QuadCore Intel Core i5 processor.

\subsection{Results and Analysis}
%Our method has one important parameter: the regularization parameter $\lambda$.

The clustering quality is measured by normalized mutual information (NMI)~\cite{lancichinetti2009detecting}, a measurement of similarity from information theory, and clustering accuracy (AC), the percentage of items correctly clustered with the maximum bipartite matching~\cite{10.1145/860435.860485}. %which is defined as follows:
%\begin{equation}
%    AC = \frac{\sum_{i=1}^{n}\delta(\alpha_i, map(l_i))}{n},
%\end{equation}
%where $n$ denotes the total number of samples, $\delta(x,y)$ equals 0 when $x=y$ and equals 1 otherwise, $map(l_i)$ is the mapping function based on maximum bipartite matching.  
AC is defined as follows:
\begin{equation}
	\label{eq:ac}
	AC = \frac{\sum_{i=1}^{n}\delta(r_i, map(l_i))}{n},
\end{equation}
where  $l_i$ is the obtained cluster label, $r_i$ is the original provided label, $n$ is the number of total samples, $\delta(x,y)$ equals $1$ when $x=y$ and equals $0$ otherwise, and $map(l_i)$ is the permutation function that maps each $l_i$ to the equivalent cluster label provided via Hungarian algorithm. 

 Fig.~\ref{fig:lambda} shows how the clustering accuracy of our method varies with different values of $\lambda$. The performance is not changing dramatically with respect to the parameter $\lambda$, and our method has consistently ideal performance if $\lambda$ is within a reasonable range. It's reasonable to observe that the optimal value of $\lambda$ is slightly dependent on the data since the dimensionality and magnitude of the data can all have some effect on it. Experiment results of normalized mutual information and clustering accuracy on the four datasets are shown in Table~\ref{tab:nmi} and Table~\ref{tab:result}. %~\ref{tab:coil}, \ref{tab:r21578}, \ref{tab:tdt2}, \ref{tab:cifar} respectively. 
We report the mean of NMI and AC for each given cluster number $K$ over 10 test runs, the highest accuracy for each $K$ is highlighted. We can see that for both image data and text data, our proposed method can always achieve the best clustering performance among all the methods, and the improvement is significant, both in NMI and AC. One potential reason that the performance on CIFAR-10 is not that good as on COIL-20 may be the images from CIFAR-10 have more complex and varying backgrounds than the images from COIL-20. Although as the number of clusters increases, all methods' clustering performance is getting worse, our method is relatively stable with an increasing number of clusters compared to other methods. 
\section{Conclusion}\label{conclusion}
In this paper, we study the symmetric matrix factorization problem with a regularization term. We propose an efficient column-wise update rule and provide a general framework that can be extended to solve symmetric matrix factorization problems with various constraints. We prove the convergence rate with theoretical analysis. The results of extensive experiments on real-world data sets validate the effectiveness of our algorithm and its superiority in data clustering tasks.

\bibliographystyle{plain}
\bibliography{icdm22}
\end{document}